\documentclass[twoside,11pt]{article}

\usepackage{blindtext}
\usepackage{amsmath}
\usepackage{tcolorbox}
\usepackage{algorithm}
\usepackage{bm}
\usepackage{algorithmic}
\usepackage{amsthm} 
\usepackage{dsfont}
\usepackage{amsfonts} 
\newtheorem{assumption}{Assumption}
\usepackage{natbib}
\usepackage{graphicx}
\usepackage{natbib}
\usepackage{booktabs}
\usepackage{float}
\usepackage{multirow}
\usepackage{subcaption}
\tcbuselibrary{skins}
\tcbset{
    generator_box/.style={
        enhanced,
        colback=yellow!10,
        colframe=yellow!50!black,
        fonttitle=\bfseries,
        title=Generator Prompt,
        attach boxed title to top left={yshift=-\tcboxedtitleheight/2},
        boxed title style={size=small,colback=yellow!50!black,colframe=yellow!50!black},
        coltitle=white,
    },
    discriminator_box/.style={
        enhanced,
        colback=blue!10,
        colframe=blue!50!black,
        fonttitle=\bfseries,
        title=Discriminator Prompt,
        attach boxed title to top left={yshift=-\tcboxedtitleheight/2},
        boxed title style={size=small,colback=blue!50!black,colframe=blue!50!black},
        coltitle=white,
    }
}
\usepackage[preprint]{jmlr2e}
\newcommand{\argmax}{\operatornamewithlimits{arg\,max}}

\usepackage{lastpage}
\begin{document}

\title{ALIGN: Aligned Delegation with Performance Guarantees for Multi-Agent LLM Reasoning}

\author{\name Tong Zhu\thanks{Equal Contribution.} \email toz015@ucla.edu \\
       \addr Department of Biostatistics, UCLA
       \AND
       \name Baiting Chen\footnotemark[1] \email brantchen@ucla.edu \\
       \addr Department of Statistics and Data Science, UCLA
       \AND
       \name Jin Zhou \email jinjinzhou@ucla.edu\\
     \addr Department of Biostatistics, UCLA
       \AND
       \name Hua Zhou \email  huazhou@ucla.edu \\
       \addr 
       Department of Biostatistics, UCLA
       \AND
       \name Sriram Sankararaman \email sriram@cs.ucla.edu \\
       \addr Department of Computer Science, UCLA
       \AND
       \name Xiaowu Dai\footnotemark[2] \email daix@ucla.edu \\
       \addr Departments of Statistics and Data Science, and of Biostatistics, UCLA}
       
\maketitle     
\renewcommand{\thefootnote}{\fnsymbol{footnote}}
\footnotetext [2] {\textit{Address for correspondence:} Xiaowu Dai, Department of Statistics and Data Science, UCLA, 8125 Math Sciences Bldg \#951554,  Los Angeles, CA 90095, USA. Email: daix@ucla.edu.}
%\editor{My editor}

\begin{abstract}\label{sec:abstract}
LLMs often underperform on complex reasoning tasks when relying on a single generation-and-selection pipeline. Inference-time ensemble methods can improve performance by sampling diverse reasoning paths or aggregating multiple candidate answers, but they typically treat candidates independently and provide no formal guarantees that ensembling improves reasoning quality.
We propose a novel method, \emph{Aligned Delegation for Multi-Agent LLM Reasoning} (ALIGN), which formulates LLM reasoning as an aligned delegation game. In ALIGN, a principal delegates a task to multiple agents that generate candidate solutions under designed incentives, and then selects among their outputs to produce a final answer. This formulation induces structured interaction among agents while preserving alignment between agent and principal objectives.
We establish theoretical guarantees showing that, under a fair comparison with equal access to candidate solutions, ALIGN provably improves expected performance over single-agent generation. Our analysis accommodates correlated candidate answers and relaxes independence assumptions that are commonly used in prior work. Empirical results across a broad range of LLM reasoning benchmarks consistently demonstrate that ALIGN outperforms strong single-agent and ensemble baselines.
\end{abstract}

\section{Introduction}
Large language models (LLMs) have demonstrated notable generalization and reasoning abilities across diverse tasks involving language understanding, generation, and decision making. \citep{achiam2023gpt, koike2024outfox,chiang2024enhancing,han2025token}. However, they continue to face challenges in complex reasoning problems that require multi-step reasoning, where problems cannot be solved in a single generation but must be decomposed into subproblems and integrated through intermediate results \citep{mirzadeh2024gsm, stechly2024self}. For example, solving mathematical problems often involves a chain of deductions in which each step depends on the previous one \citep{imani2023mathprompter}. A prominent line of work focuses on inference-time prompting strategies, which guide models to generate diverse reasoning paths, such as chain-of-thought prompting \citep{wei2022chain} or sampling multiple rationales \citep{guan2025rstar}, followed by selection using reward models or verifiers \citep{lightman2023let}. However, these methods treat each reasoning path independently, without principled mechanisms for interaction or refinement, and their effectiveness remains limited by the capacity of individual models \citep{spraguemusr, xu2024pride, stechly2024self, wang2022self}.

% Recent efforts to improve reasoning have mainly taken two forms:
% (1) fine-tuning with data or human feedback, which directly adjusts model parameters \citep{sun2023aligning,wang2024mathcoder}; and
% (2) inference-time improvements, such as prompting diverse reasoning paths \citep{wei2022chain, guan2025rstar} or enabling self-improvement through reflection, critique, and refinement \citep{madaan2023self, cheng2024self}.

Recent work has explored more interactive inference-time strategies for improving LLM reasoning. Self-reflection and critique-based methods encourage iterative refinement of model outputs \citep{madaan2023self,cheng2024self}, while ensemble approaches leverage multiple LLMs engaging in debate, feedback exchange, or negotiation to improve answer quality \citep{huang2024ensemble,chen2025incentivizing,wang2024mobile}. Despite empirical improvements, these methods provide no formal guarantees that refinement or ensembling consistently improves reasoning quality. In addition, ensemble approaches typically rely on each participating model having strong capacity, which limits their effectiveness when individual models are weaker.

We propose \emph{Aligned Delegation for Multi-Agent LLM Reasoning} (ALIGN), a training-free, game-theoretic framework for improving LLM reasoning \emph{without} additional fine-tuning or task-specific retraining. In ALIGN, multiple \emph{agent LLMs} independently generate candidate answers and submit one for evaluation, while a \emph{principal LLM} ranks the submissions and selects the final output. The utility of each agent is determined by its internal utility over the answer and its relative ranking given by the principal, so each agent must balance its internal utility with the likelihood of receiving a favorable ranking from the principal. This delegation-based structure ensures that agents are incentivized to improve their answers in ways that align with the principal's evaluation criteria.

To implement this framework, we use online mirror descent to iteratively update the policy of each agent, allowing the system to converge to equilibrium.  Conceptually, ALIGN draws intuition from competition-driven systems studied across disciplines: in evolutionary settings, competition among diverse entities promotes adaptation and higher-quality outcomes \citep{endler1986natural,albadr2020genetic}, while in economics, market competition among self-interested agents can improve aggregate efficiency \citep{podolny1993status,gupta2016marketing}. In ALIGN, analogous competitive pressures among agents incentivize the exploration of higher-quality reasoning strategies under aligned objectives. This collective dynamic enables a scalable approach to improving inference-time reasoning without requiring golden answers. Empirical results across multiple reasoning benchmarks show that ALIGN consistently outperforms strong single-agent and ensemble baselines.

To summarize, our contributions are threefold:
\begin{itemize}
\item We formalize inference-time multi-agent LLM reasoning as a delegation game and provide a learning algorithm based on online mirror descent.
\item We establish theoretical guarantees for delegation-based reasoning under a fair comparison, showing strict performance improvement over single-agent generation while allowing for correlated candidate answers. We further characterize the learning dynamics by proving sublinear regret and convergence of average policies to a Nash equilibrium.
\item We present extensive empirical evaluations across diverse reasoning benchmarks that validate the theoretical findings and quantify the improvement of delegation-based inference relative to strong generation-and-selection and ensemble baselines.
\end{itemize}

%%%%%%%%%%%%%%%%%%%%%%%%%%%%%%%%%%%%%%%%%%%%%%%%%
%\subsection{Extentions}
%Sriram's suggestions on (i) genetic benchmarks; (ii) correct paper citations. 

%%%%%%%%%%%%%%%%%%%%%%%%%%%%%%%%%%%%%%%%%%%%%%%%
\section{Methodology}\label{sec:method}

\subsection{Motivation}\label{sec:motivation}
LLMs are observed to fall short on complex tasks such as mathematical reasoning, multi-step planning, and commonsense reasoning  \citep{mirzadeh2024gsm,stechly2024self,kwon2024toward}. One promising direction to address these limitations is to encourage a \emph{single} LLM to think like humans by generating multiple reasoning paths and selecting the most plausible answer, analogous to how humans tackle complex problems through diverse strategies \citep{yao2023react}. These are often coupled with a verifier model or selection criterion to identify and retain the most consistent or accurate response \citep{guan2025rstar,cobbe2021training}.
Specifically, for a given task \(t \in \mathcal{T}\), an LLM may generate multiple candidate answers \(\{\omega_1, \ldots, \omega_K\}\) either through repeated sampling or by encoding diverse reasoning paths in prompts \citep{guan2025rstar,cobbe2021training}. The final answer is selected by maximizing utility \(U: \mathcal{A} \to \mathbb{R},\)
where $\mathcal{A}$ denotes the set of admissible answers, i.e., the LLM output space.  The utility $U$ can be instantiated as an external reward model or derived from internal evaluations such as consistency or factuality checks \citep{zhang2024generative,guan2025rstar}. That is, the selected answer is given by
\(
\omega^* = \argmax_{\omega \in \{\omega_1, \ldots, \omega_K\}} U(\omega). 
\)

However, this \emph{single}-LLM approach faces two main challenges. First, answer selection can be biased and unreliable, since reward models or self-evaluations may not capture true task quality \citep{zheng2024large,wang2024large}. Second, candidate answers are usually treated in isolation,  ignoring potential complementarities or cross-validation among them.  To address these issues, we propose to leverage \emph{multi-agent} LLMs, where multiple models act as competing agents and provide diverse perspectives that mitigate bias and improve robustness in the selection process.

\subsection{Aligned Delegation for Multi-Agent LLM} \label{sec:delegation game}
We reinterpret the reasoning-and-selection process, where an LLM generates candidate responses that are later filtered into a final output, as an \emph{aligned delegation game}. In this game-theoretic view, a principal delegates tasks to self-interested agents with potentially misaligned preferences and then selects from their responses to achieve desirable outputs
\citep{fershtman1991observable, frankel2014aligned, guo2016dynamic}.
Building on this perspective, we model an aligned delegation game where agents, each following a distinct reasoning strategy, submit candidate responses, and a principal acts as centralized evaluator to rank them.
We term this setup the Aligned Delegation for Multi-Agent LLM (ALIGN), highlighting how principal feedback induces competition among heterogeneous multi-agents.

Specifically, for each task $t \in \mathcal{T}$ and each agent $i \in [N]$, the agent generates a set of candidates
\(
\mathcal{A}_i=\{\omega_{i1},\ldots,\omega_{iK}\}\subseteq \mathcal{A},
\)
then selects a submission $a_i \in \mathcal{A}_i$ according to its internal utility $U_{y_i}:\mathcal{A}\to\mathbb{R}$. This utility is operationalized via \emph{self-consistency} \citep{wang2022self}: 
\(
U_{y_i}(a) := \mathbb{P}_{i} \left( a | t \right) \approx \frac{1}{K} \sum_{k=1}^K \mathbb{I}[\omega_{ik} = a],
\)
where the probability is estimated by the empirical frequency of \(a\) among the sampled responses. Intuitively, higher utility is assigned to answers the agent would generate more consistently. Then, the principal aggregates the submissions $\{a_1,\ldots,a_N\}$ and evaluates them via a global utility $U:\mathcal{A}\to\mathbb{R}$  aligned with user preferences. It is shaped by implicit or explicit user instructions via the prompts, with the principal acting as a proxy for the user, selecting the response that best aligns with the user’s intended objective.
After evaluation, each agent $i$ receives a feedback $r_i \in \mathbb{R}$ determined by the relative ranking of $a_i$ (e.g., $r_i=+1$ for top rank, $r_i=-1$ for bottom rank, and intermediate values otherwise). This design encourages agents to explore diverse reasoning paths, but it also increases the risk of misalignment between their local objectives and the principal’s utility. To reconcile the two, each agent $i$’s reward is defined as,
$U_i(a_i) = r_i \cdot U_{y_i}(a_i)$, 
which combines its internal preference with the ranking feedback from the principal.
The ranking-based mechanism creates structural tension: agents must balance their internal reasoning with the principal’s evaluation, leading them to refine strategies that improve both their individual quality and their relative standing.
The overall process of ALIGN is illustrated in Figure~\ref{fig:principal_agent}.

% The utility functions of principal and agents reflect their respective preferences over possible outputs. For each agent \(i\), its internal utility function \(U_{y_i}(a)\) reflects the agent's internal preference for answer \(a\).
%  This preference is operationalized via \emph{self-consistency} \citep{wang2022self}: the agent generates multiple responses \(\{\omega_{i1}, \ldots, \omega_{iK}\}\) to the same input and defines
% \[
% U_{y_i}(a) := \mathbb{P}_{i} \left( a | t \right) \approx \frac{1}{K} \sum_{k=1}^K \mathbb{I}[\omega_{ik} = a],
% \]
% where the probability is estimated by the empirical frequency of \(a\) among the sampled responses. Intuitively, higher utility is assigned to answers the agent would generate more consistently. In contrast, the principal’s utility function \( U_x(a) \) reflects user-centered preferences. It is shaped by implicit or explicit instructions from the users and serves as an external evaluation criterion over the submitted responses. In this sense, the principal acts as a proxy for the user, selecting the response that best aligns with the user’s intended objective or value.
\begin{figure*}[htbp]
    \centering
  \includegraphics[width=\textwidth]{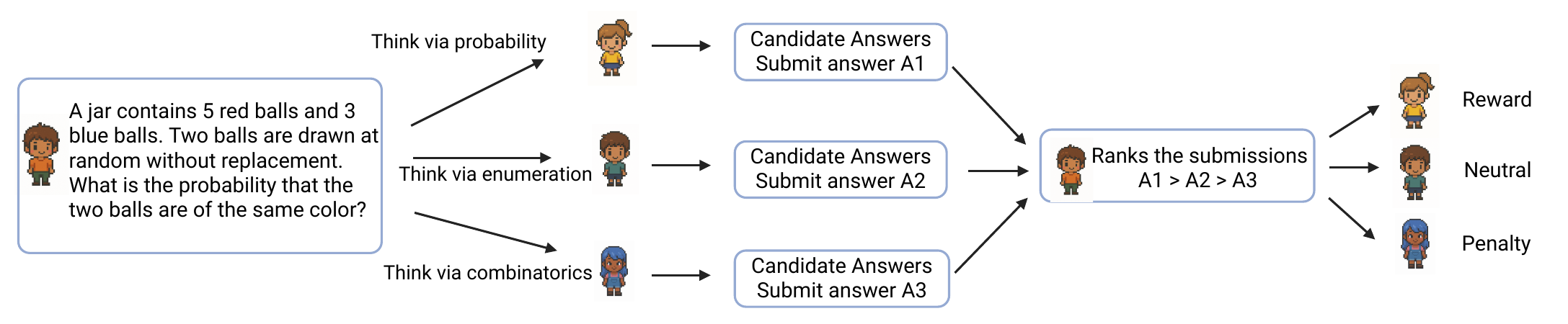} 
    \caption{Overview of ALIGN: multiple agents propose answers along distinct reasoning paths and receive utility from the principal’s ranking feedback. This setup incentivizes high-confidence outputs while promoting alignment with the principal.}
    \label{fig:principal_agent}
\end{figure*}

\subsection{Implementation of ALIGN}\label{sec:algorithm}

\begin{algorithm}[t]
\caption{ALIGN with Mirror Descent}
\label{algorithm}
\begin{algorithmic}[1]
    \STATE \textbf{Initialize:} For each agent \( i \in [N] \), set candidate set \( \mathcal{A}_i \), initialize utility estimates \( U_i^0(a) \gets 0 \) and policy \(\pi_i^0(a)\) for all \( a \in \mathcal{A}_i \), choose learning rate \( \eta > 0 \).
    \FOR{each round \( t = 1, 2, \ldots \)}
        \FOR{each agent \( i \in [N] \) (in parallel)}
            \STATE Sample answer \( a_i^t \sim \pi_i^t \) and submit \( a_i^t \) to principal
        \ENDFOR
        \STATE Principal ranks the submitted answers and provides feedback \(r_i^t\)
        % \(
        % a^* = \arg\max_{a_i^t} U(a_i^t)
        % \)
        \FOR{each agent \( i \in [N] \)}
            \FOR{each candidate answer \( a \in \mathcal{A}_i \)}
                \STATE Update utility:\\
                \quad $U_i^t(a) \gets U_i^{t-1}(a) +  r_i^tU_{y_i}(a,a^t_{-i})$
            \ENDFOR
            \STATE Compute policy: \\
            \quad $\pi_i^t(a) \propto \exp\left\{ \eta U_i^{t-1}(a) \right\}, \quad \forall a \in \mathcal{A}_i$
        \ENDFOR
    \ENDFOR
\end{algorithmic}
\end{algorithm}

In the implementation, ALIGN consists of the following steps. (i) Given a task \( t \in \mathcal{T} \), each agent generates a set of candidate answers. (ii) Each agent selects one answer to submit, according to a policy that defines a probability distribution over its candidates, estimated via self-consistency from repeated sampling. (iii) The principal evaluates and ranks the submitted answers according to its own utility. (iv) This ranking serves as feedback, where each agent receives a scalar utility based on the relative position of its answer. (v) Each agent updates its policy based on this feedback, with the goal of improving future performance while competing against others.

For agent \( i \), we define a policy \( \pi_i \) as a probability distribution over its candidate set \( \mathcal{A}_i \), where \( \pi_i(a) \) denotes the probability of selecting an answer \( a \in \mathcal{A}_i \). The agent’s objective is to adaptively update this policy to increase its expected cumulative utility, given the outcomes of prior interactions.
We adopt a mirror descent update rule, which adjusts the policy by shifting probability mass toward answers with higher observed utility, while maintaining exploration through regularization \citep{duvocelle2023multiagent,jacob2022modeling}. Specifically, mirror descent provides a principled framework for updating distributions over actions, and in our setting corresponds to an exponential weighting scheme based on utility feedback \citep{shalev2012online}. When using negative entropy as the mirror map, the resulting update recovers the classical Hedge algorithm from online learning \citep{littlestone1994weighted}. The full procedure is described in Algorithm~\ref{algorithm}.

\section{Theoretical Guarantees}\label{sec:theory}
\subsection{Performance Guarantees for ALIGN}
\label{sec:delegation theory}

We establish that multi-agent delegation can yield better performance than the single-agent setting. To enable a fair comparison, we ensure both settings have equal total access to candidate answers. In the single-agent case, the agent draws \(k\) answers from a distribution \(D\) to form its candidate set. In the multi-agent case, we consider  \(N \geq 2\) agents, where each agent \(i\) independently draws \(k_i\) samples from \(D\), with the same total number of samples: \(\sum_{i=1}^N k_i = k\). We assume all agents share the same internal utility function, \(U_{y_i} = U_y\), and that the principal applies the same evaluation function \(U\) across both settings. Under these conditions, any performance improvement arises solely from delegation rather than unequal information.

% \begin{assumption}[Pareto-optimal play]\label{assump:pareto}
% If an agent has two solutions $\omega$ and $\omega'$ such that
% \[
% U_x(\omega) \geq U_x(\omega') \quad \text{and} \quad U_y(\omega) \geq U_y(\omega'),
% \]
% and at least one of these inequalities is strict, then the agent does not select $\omega'$ to the principal.
% \end{assumption}

% \begin{assumption}[Symmetric agents]\label{assp:symmetric}
% We say that the agents are \emph{symmetric} if each agent provides the same number of answers and utility of each answer follows the same distribution \(D\).
% \end{assumption}

% \begin{assumption}\label{assp:corr}
% For any agent $i \in [n]$, the utility functions $U_x(\cdot)$ and $U_{y_i}(\cdot)$ are not negatively correlated, i.e.,
% \[
% \mathrm{Corr}\big(U_x(\omega), U_{y_i}(\omega)\big) \ge 0, \quad \text{for } \omega \sim D_i.
% \]
% \end{assumption}

\begin{assumption}\label{assump:structure}
\textbf{(i) Pareto-optimal play.} If an agent has two candidate answers \( \omega \) and \( \omega' \) such that the principal’s utility \( U(\omega) \geq U(\omega') \) and the agent’s utility \( U_{y_i}(\omega) \geq U_{y_i}(\omega') \), with at least one inequality being strict, then the agent does not submit \( \omega' \) to the principal.

\textbf{(ii) Symmetric agents.} All agents generate the same number of candidate answers, and the utility of each answer follows the same distribution \( D \).

\textbf{(iii) Non-negative alignment.} The principal's utility  \( U(\cdot) \) are not negatively correlated with any agent's utility \( U_{y_i}(\cdot) \), i.e.,
$\mathrm{Corr}\big(U(\omega), U_{y_i}(\omega)\big) \ge 0,  \text{ for } \omega \sim D$.
\end{assumption}

\noindent
Part (i) of Assumption~\ref{assump:structure} enforces Pareto-optimal play: if one answer is at least as good for both the agent and the principal, and strictly better for at least one, then the agent will not choose the inferior one. 
Part (ii) ensures that all agents operate under comparable conditions. Each agent generates the same number of candidate answers and submits one to the principal, so no agent gains an advantage from producing more options. By using LLMs of similar capacity with identical sampling procedures, the distribution of generated candidates is symmetric across agents. 
Part (iii) rules out adversarial behavior by ensuring that when the principal values an answer, agents do not systematically devalue it. 

A detailed discussion of the necessity of Assumption~\ref{assump:structure} is provided in Appendix~\ref{appendix:necessity assumptions}. Without enforcing part (ii), a misaligned “super agent” could strategically submit low-utility answers that the principal is forced to accept, driving the principal’s expected utility arbitrarily close to zero even when substantially better answers exist. Without part (iii), when agents’ utilities are negatively correlated with that of the principal, agents act selfishly and adversarially, submitting answers that undermine the principal’s objective. We also empirically verify parts (i) and (iii) in Section~\ref{sec:assump validation}.

% \begin{theorem}\label{thm:multi-agent adv}
% Given a single-agent problem $P$ and its multi-agent correspondence $P'$, 
% for any mechanism $M$ under $P$, there exists a multi-agent single proposal mechanism $M'$ under $P'$ such that, 
% at the Nash equilibrium of each mechanism, the principal's utility under $P'$ is at least that under $P$:
% \[
% U_x(M') \geq U_x(M).
% \]
% \end{theorem}

% The main take-away is that the principal can increase (or at least not decrease) her utility by recruiting more agents with similar utilities. Intuitively, in equilibrium multi-proposal and single-proposal are equivalent, and multiple agents each making a single proposal give the principal more chances to receive a highly valued solution, as more judgments lead to better outcomes.

% \begin{theorem}\label{thm:mspm-dominance}
% In the multi-agent setting, for any mechanism $M$ with an arbitrary equilibrium $\sigma$, there exists a multi-agent single proposal $M'$ such that for any equilibrium $\sigma'$ of $M'$, for any observed solutions $\bar{\omega} \in \Omega^*$, and for any
% \(
% \omega \in F_{M,\sigma}(\bar{\omega})\) and \(\omega' \in F_{M',\sigma'}(\bar{\omega}),
% \)
% we have
% \[
% x(\omega) \leq x(\omega').
% \]
% \end{theorem}

To evaluate the efficiency of mechanisms in strategic settings, 
we compare their outcomes with the expected value of the optimal result, 
denoted as $\mathbb{E}[U_{\max}]$, 
where $U_{\max}$ represents the principal's utility from the best candidate answer 
among those generated by the agent. A mechanism $M$, under agent strategies $\sigma$, is said to be 
$(\rho, \gamma)$-approximate if its expected outcome satisfies
\(
\rho\mathbb{E}[U_{M,\sigma}] + \gamma \;\geq\; \mathbb{E}[U_{\max}] ,
\)
where $\mathbb{E}[U_{M,\sigma}]$ denotes the principal's expected utility 
from the answer selected by $M$ under strategies $\sigma$, 
and $\rho$ and $\gamma$ are the multiplicative and additive 
approximation factors, respectively.
In our setting, we focus on prior-independent mechanism, where the mechanism has no prior knowledge of the distributions from which the agents’ answers are drawn. Additionally, we adopt an incomplete information framework among agents: while agents can be aware of the principal's utility function, 
they do not observe each other's submitted answers but only observe the ranking feedback returned by the principal.

% Building on this, we define the \emph{price of anarchy} (PoA) as a measure of the worst-case efficiency loss due to strategic behavior. Specifically, the multiplicative price of anarchy $\text{PoA}_m$ is the smallest $\rho$ such that the mechanism is $(\rho, 0)$-approximate under \textbf{every} Nash equilibrium. Similarly, the additive price of anarchy $\text{PoA}_a$ is the smallest $\gamma$ such that the mechanism is $(1, \gamma)$-approximate under all equilibria. 
% In contrast, the \emph{price of stability} (PoS) captures the best-case performance at equilibrium: the multiplicative price of stability $\text{PoS}_m$ is the smallest $\rho$ such that there exists \textbf{some} Nash equilibrium under which the mechanism is $(\rho, 0)$-approximate, and the additive price of stability $\text{PoS}_a$ is similarly defined for $(1, \gamma)$-approximation. 

% not exact
\begin{theorem}\label{thm:multi-agent-delegation}
Consider a single-agent problem $P$ and its multi-agent correspondence $P'$ with $N$ agents. Then under Assumption \ref{assump:structure}, we have,
\begin{itemize}
    \item[(a)] For any mechanism $M$ under $P$, there exists a multi-agent single-proposal mechanism $M'$ under $P'$ such that, at the Nash equilibrium of each mechanism,
\(
U(M') \;\geq\; U(M),
\)
where $U(\cdot)$ denotes the principal's utility function.
\item[(b)] When each agent $i$ generates candidate answers independently 
and the principal utility of these candidates follow $\mathcal{U}[-1,1]$, 
and the agent is willing to tolerate a utility loss of at most $2\varepsilon$ 
relative to its optimal candidate. 
That is, instead of always selecting the answer that maximizes its own utility 
$U_i^{\max}$, the agent may strategically submit any answer whose utility is at least 
$U_{i,\max} - 2\varepsilon$.
As a result, a \(2\varepsilon\)-approximate Bayes-Nash equilibrium can be achieved and the expected utility attained by the principal satisfies
$\mathbb{E}[U_{M',\sigma'}] = \mathbb{E}[U_{\max}]$,
where 
\(
\varepsilon \leq 1 - e^{- \frac{N^2}{2(N-1)^2}}.
\) 
\end{itemize}
\end{theorem}

Part (a) of this theorem shows that the principal can obtain a higher-utility answer by recruiting more comparable agents. Intuitively, when multiple agents each submit one answer, the principal has more chances to receive a better answer, as independent judgments may increase the likelihood of a better outcome.
Part~(b) further shows that, if each agent is not purely self-interested and is willing to consider answers with utility at least \(U_{i,\max}-2\varepsilon\), then the principal can obtain the best answer with utility \(\mathbb{E}[U_{\max}]\).
Unlike prior work that assumes independent utilities \citep{shin2023delegating}, our setting involves a shared problem between the principal and agents, making independence unrealistic. We therefore extend the analysis to positively correlated utilities (Assumption~\ref{assump:structure}), which more accurately capture the structural alignment between LLM agent and principal objectives.

\subsection{Regret Analysis}\label{sec:convergence}
We analyze the learning dynamics of our algorithm in the multi-agent setting and establish regret guarantees for each agent. Regret is evaluated with respect to each agent’s own utility and measures the gap between the utility the agent actually obtains and the utility it could have obtained by following the best fixed policy in hindsight. Specifically, we show that each agent achieves sublinear regret over time when running Algorithm~\ref{algorithm}, so its cumulative utility asymptotically matches that of its best fixed policy.

% To do so, we first reformalize the utility function. The actual utility received by each agent depends on both the principal's feedback and the agent's internal utility. Specifically, in round \( t \), the agent receives a scalar feedback signal \( r_i^t \in \mathbb{R} \) based on the ranking of its submitted answer, and the utility associated with action \( a \in \mathcal{A}_i \) is defined as
% \(
% U_i^t(a) = r_i^t \cdot U_{y_i}(a),
% \)
% where \( U_{y_i}(a) \) denotes the agent's internal valuation of action \( a \). This utility formulation captures the interaction between the external evaluation signal and the agent’s own reasoning preferences.

For agent \( i \) and any candidate answers \( a \in \mathcal{A}_i \), we define the regret after \( T \) rounds as the average difference between the cumulative utility the agent would have received by consistently selecting an answer \( a \), and the utility actually obtained by following the sequence of mixed strategies \( \{ \pi_i^t \}_{t=1}^T \) \citep{cai2023doubly}.
Formally, the regret is defined as
\[
R_i^T(a) = \max_{\pi} \left\{ \sum_{t=1}^T [U_i(\pi, a_{-i}^t) - U_i(\pi_{i}^t, a_{-i}^t)] \right\},
\]
where $a_{-i}^t$ denotes the answers chosen by all agents other than $i$ at round $t$, and $u_i(\pi_i^t, a_{-i}^t) = \sum_{a' \in \mathcal{A}_i} \pi_i^t(a')\, U_i(a', a_{-i}^t)$ is the expected utility of agent $i$ under their mixed strategy $\pi_i^t$. 
The regret quantifies, for each fixed answer $a$, how much worse the agent performed on average compared to always selecting that answer. The agent is said to have \emph{no regret} if $\max_{a \in \mathcal{A}_i} R_i^T(a)/T \to 0$ as $T \to \infty$ \citep{bubeck2012regret}.
More generally, one can express the cumulative regret with respect to any fixed mixed strategy $\pi$, in which case the following upper bound holds:

% \begin{theorem}
% Upon running Algorithm~\ref{algorithm} for $T$ iterations in a multi agents general-sum game, the policy $\bar{\pi}_i^T = \frac{1}{T}\sum_{t=1}^T\pi_i^t$ is at a distance
% \[
% D_{\mathrm{KL}}(\bar{\pi}_i^T \,\|\, \tau_i) \leq \frac{1}{\lambda_i} \left( \frac{R_i^T}{T} + D_i \right),
% \]
% where $D_i$ is any upper bound on possible rewards for Player~$i$, $\lambda_i$ is the regularization term and \(
% R_i^T := \max_{\pi^* \in \Delta(\mathcal{A}_i)} \sum_{t=1}^{T} U_i^t(\pi^*) - \sum_{t=1}^{T} U_i^t(\pi_i^t)
% \) is the regret for agent \(i\). 
% \end{theorem}

\begin{theorem}\label{thm:no regret}
Let agent \( i \) follow Algorithm~\ref{algorithm} with learning rate \( \eta = 1/\sqrt{T} \). Then for any policy \( \pi \), the cumulative regret over \( T \) rounds satisfies
$\sum_{t=1}^T [ U_i(\pi, a_{-i}^t) - U_i(\pi_i^t, a_{-i}^t) ]
\leq ( \frac{1}{4} + D_{\mathrm{KL}}(\pi \,\|\, \pi_i^0) ) \sqrt{T}$.
\end{theorem}

Theorem~\ref{thm:no regret} shows that when each agent follows Algorithm~\ref{algorithm} with learning rate \( \eta = 1/\sqrt{T} \), the cumulative regret after \( T \) rounds is upper bounded by \( \mathcal{O}(\sqrt{T}) \). The factor in the bound depends on two components: a fixed coefficient \( 1/4 \) from the optimization dynamics, and the KL divergence between the comparator policy and the initial policy. As a result, the average regret  $\sum_{t=1}^T \left[ U_i(\pi, a_{-i}^t) - U_i(\pi_i^t, a_{-i}^t) \right]/T$ vanishes as \( T \to \infty \), which ensures that each agent learns to perform competitively over time. This guarantee is consistent with standard results in online learning and mirror descent algorithms \citep{cai2023doubly,jacob2022modeling}.

\subsection{Convergence Guarantees}
We now analyze the convergence behavior of agents to a Nash
equilibrium. In our setting, a Nash equilibrium corresponds to a stable outcome in which each agent adopts the best policy over candidate answers that maximizes its expected utility, given the fixed policies of all other agents \citep{kreps1989nash, holt2004nash}. Formally, let \( \pi = (\pi_1, \ldots, \pi_N) \) denote the joint policy profile, where each \( \pi_i  \) is a probability distribution over the candidate set \( \mathcal{A}_i \) of agent \(i\). Then \( \pi^* = (\pi_1^*, \ldots, \pi_N^*) \) is a Nash equilibrium if for all \( i \in N \) and any alternative policy \( \hat{\pi}_i  \),
\[
U_i(\pi_i^*, \pi_{-i}^*) \geq U_i(\hat{\pi}_i, \pi_{-i}^*).
\]
At equilibrium, no agent can improve its expected utility by changing its policy alone.

\begin{theorem}\label{thm:convergence}
For any $T \in \mathbb{N}$, $\eta > 0$, and $\delta \in (0,1)$, define the quantity
\[
\xi^T(\delta) := \frac{\sum_{i=1}^N R_i^T}{T} + N \sqrt{ \frac{8}{T} \log \left( \frac{N \max_i |\mathcal{A}_i|}{\delta} \right) }.
\]
For \(N\)-agents delegation games, upon running Algorithm~\ref{algorithm} for any $T$ iterations with learning rate $\eta > 0$, the average policies $\bar{\pi}_i^T = \frac{1}{T} \sum_{t=1}^T \pi_i^t$ of each agent form a $\xi^T(\delta)$-approximate Nash equilibrium with probability at least $1 - \delta$, for any $\delta \in (0,1)$.
\end{theorem}

Theorem~\ref{thm:convergence} establishes that the empirical average of the policies produced by Algorithm~\ref{algorithm} converges toward equilibrium behavior. Specifically, the average strategy profile $(\bar{\pi}_1^T,\dots,\bar{\pi}_N^T)$ is guaranteed to be a $\xi^T(\delta)$-approximate Nash equilibrium with high probability. The error term $\xi^T(\delta)$ decomposes into two parts: (i) the cumulative regret $\sum_i R_i^T / T$, which vanishes sublinearly under Algorithm~\ref{algorithm}, and (ii) a concentration term of order $O(\sqrt{\log(N \max_i |\mathcal{A}_i|/\delta)/T})$ arising from standard martingale inequalities. Together, these imply that the approximation error decays at the rate $O(1/\sqrt{T})$, ensuring that no player can improve her long-run utility by more than $O(1/\sqrt{T})$ through unilateral deviation.

The equilibrium reflects the outcome of repeated learning dynamics where agents iteratively adapt their policies to better align with the principal’s feedback while preserving their own reasoning-based preferences. In this sense, each agent’s policy converges to a distribution over answers that balances two forces: being favored by the agent’s internal utility and being ranked highly by the principal. The result generalizes the convergence guarantees of \citet{jacob2022modeling} from two-player zero-sum games to general multi-agent aligned delegation settings.

\section{Experiments}
\label{sec:experiments}
We evaluate ALIGN on mathematical reasoning question-answering (QA) tasks. In this setup, a principal LLM provides ranking feedback, while three agent LLMs independently generate candidate answers and compete to obtain higher rankings.
 Code for all experiments is available at \url{https://anonymous.4open.science/r/aligned_delegation_algorithm-7122}.

\subsection{Experiment Setup}\label{Sec: Exp Setup}

\textbf{Models}. We adopt four open-sourced instruction-following LLMs to serve as agents in ALIGN: Mistral-8B-Instruct \citep{mistral2024ministral}, Zephyr-7B-Beta \citep{tunstall2023zephyrdirectdistillationlm} and Phi-3-Mini-4K-Instruct \citep{abdin2024phi}. These models have comparable parameters and are matched in capacity, satisfying Assumption~\ref{assump:structure}, while contributing architectural and training diversity to the agent pool. We use Qwen/Qwen2.5-7B-Instruct ~\cite{yang2025qwen3} as the principal because of its good alignment and consistent performance on all of our reasoning benchmarks. Unless otherwise specified, we set the learning rate $\eta = 0.1$ and 20 maximum iteration number for all experiments.

\textbf{Implementation Details.} In the candidate answer generation stage, we augment each agent with Monte Carlo Tree Search (MCTS) , performing 16 roll-outs with 5 maximum depth. In the delegation stage, a Llama2-7b model \citep{touvron2023llama} served as the principal, providing feedback via best-path masking and consistency checks to the agents. The prompts used throughout the system adhered to the format described in the work by \cite{guan2025rstar}. 

\textbf{Baselines.} We compare our delegation-based reasoning framework against three strong and representative baselines: 
\begin{itemize}
    \item \textbf{Few-shot Chain-of-Thought }\citep{wei2022chain} is a method that prompts a LLM with a few in-context exemplars illustrating intermediate reasoning steps before the final answer. This baseline reflects the model’s standalone reasoning ability without delegation.
    \item \textbf{Self-Consistency with Chain-of-Thought} (CoT-SC) \cite{wang2022self}, where we sampled the answers 16 times, employing majority voting for the selection of the answers.
    
    \item \textbf{rStar} \citep{guan2025rstar} is a self-play approach that improves reasoning through a generation and discrimination process. It first generates multiple reasoning paths and then uses a discriminator LLM to filter answers. In our implementation, the principal LLM fills in missing steps given earlier steps and scores candidate answers. We adopt a simplified variant that only verifies and scores the agents’ submitted answers.
%     \item \textbf{rStar generator @maj} (\cite{guan2025rstar}) is a component of rStar to generate multiple reasoning paths. We add the rStar discriminator for mutual-consistency filtering. The agents implement Monte Carlo
% Tree Search (MCTS) to five different ways to propose multiple candidate answers, and we select the final answer by majority voting across roll-outs.
\item To compare against a setting where the principal acts alone (without delegation), we also report the results selected by the principal in the final iteration; we refer to this setting as \textbf{Principal}.
\end{itemize}

\textbf{Datasets}.
We conduct our evaluations using mathematical reasoning datasets: GSM8K, MATH, and GSM-Hard \citep{cobbe2021gsm8k,lightman2023let, gao2022pal}. Each dataset presents unique challenges, allowing us to evaluate the proposed delegation framework across different types of reasoning skills and linguistic variation. Details are summarized below.
\begin{itemize}
 \item MATH \citep{lightman2023let} is a dataset of challenging competition-level mathematics problems ranging from high school to early college level. It covers diverse topics such as algebra, geometry, probability, and calculus, requiring models to demonstrate advanced mathematical reasoning and symbolic manipulation skills. 
\item GSM8K \citep{cobbe2021training} is a benchmark for grade-school math problems. It composed of 8.5K high-quality word problems that require multi-step arithmetic reasoning. The test set is about 1.3k. The dataset evaluates how well models can perform numerical computations and logical deductions in everyday language scenarios.
    
\item GSM-Hard \citep{gao2022pal} is a hard version of GSM8K math reasoning dataset. It replace the numbers in the questions of GSM8K with larger numbers that are less common.
\end{itemize}

\subsection{Experiment Results}

\textbf{ALIGN Improves Accuracy.} Results in Table~\ref{tab:accuracy} and Figure~\ref{fig:delegation_result} show that ALIGN consistently outperforms all
baselines across the evaluated datasets. This performance gain is due to two key factors. First, recruiting multiple agents expands the solution set and increases answer diversity, giving the principal a richer pool to choose from and a higher chance of including a high-quality reasoning path than few-shot CoT or the principal acting alone. Second, ALIGN connects agent utility to the principal’s ranking feedback, as agents are rewarded for producing answers that score well under the principal’s evaluation, which steers search toward faithful, well-justified solutions and induces productive competition rather than redundant exploration. 

\noindent\textbf{ALIGN Enables Competition Among Heterogeneous Agents.} Despite substantial differences in dataset difficulty and baseline accuracy, Table~\ref{tab:command} shows that most agents benefit from ALIGN. As expected, ALIGN induces competition among agents for receiving higher utilities. The ranking-based utility induces competition and a winner-takes-more effect: an initially stronger agent gains the most, while weaker agents might be mislead. For example, stronger agents with higher accuracy such as Mistral on Math and Phi on GSM-8K show some gains after competition. .

% \textbf{COMMAND Achieves Fast Convergence.} In our experiments, we set maximum iteration number 20, and early step when convergence no policy update.
\noindent\textbf{ALIGN Does Not Rely on Agent Heterogeneity.}
To directly test whether ALIGN’s gains come from multi-agent systems, we conduct a controlled experiment that disentangles the effects of agent heterogeneity. In this experiment, all agents are initially using the same backbone model, Mistral-7B-Instruct~\cite{chan2023chateval}, and differ only in decoding temperature (0.6, 0.8, and 0.9), while Qwen/Qwen2.5-7B-Instruct~\cite{yang2025qwen3} serves as the principal. The results in Table~\ref{tab:accuracy_same_model} report the accuracy for all baselines, while Table~\ref{tab:mistral_command} reports the accuracy before and after applying ALIGN for each agent. ALIGN continues to yield consistent accuracy improvements even when all agents share the same underlying model, indicating that heterogeneity and candidate count alone do not explain the observed gains.

\noindent\textbf{ALIGN is Robust to Principle Selection}
To evaluate whether ALIGN’s performance gains depend on the choice of principal model, we conduct an additional robustness study by replacing Qwen/Qwen2.5-7B-Instruct ~\cite{yang2025qwen3} with LLaMA-2-7B-Instruct \citep{touvron2023llama2openfoundation} as the principal, while keeping the same three open-source instruction-following LLMs as agents. For each benchmark, we randomly sample 300 questions. The results in Table~\ref{tab:Llama_command} report the accuracy of all baselines under this alternative principal and Table~\ref{tab:LLAMA_command2} report the individual performance of each agents befor and after ALIGN. Across all three benchmarks, ALIGN consistently achieves the best performance, outperforming rStar by 1.3 to 10.3 absolute accuracy points and substantially exceeding few-shot CoT. These findings demonstrate that ALIGN remains effective across different principal models, confirming that its performance gains are driven by delegation-based reasoning rather than reliance on a specific principal. Overall, ALIGN exhibits strong robustness to principal selection.

\begin{table}[ht]
\caption{Accuracy (\%) across benchmark datasets.
\textbf{Bold} indicates best performance.
FS-CoT: Few-shot CoT; SC: CoT-SC@maj16.}
\label{tab:accuracy}
\centering
\resizebox{.6\linewidth}{!}{%
\begin{tabular}{cccccc}
\toprule
Dataset & FS-CoT & SC & Principal & rStar & ALIGN \\
\midrule
Math     & 8.0  & 13.2 & 10.2 & 33.0 & \textbf{34.8} \\
GSM8K    & 40.4 & 62.8 & 33.5 & 57.1 & \textbf{64.3} \\
GSM-Hard & 17.7 & 24.1 & 15.6 & 22.8 & \textbf{24.6} \\
\bottomrule
\end{tabular}}
\end{table}

\begin{figure}[htbp]
    \centering
\includegraphics[width=.45\textwidth]{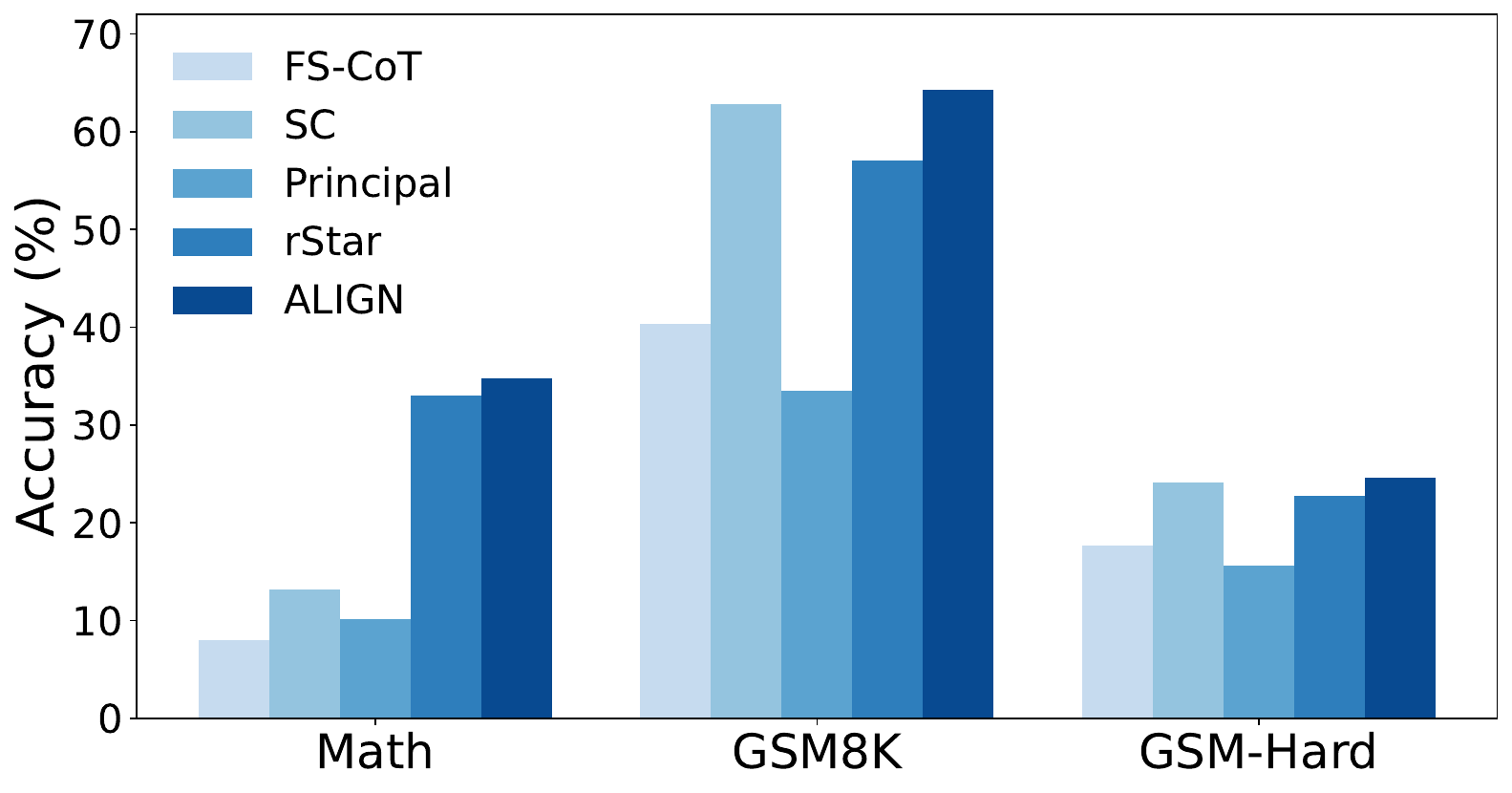} 
    \caption{Accuracy comparison across benchmark datasets for each method.}
    \label{fig:delegation_result}
\end{figure}

\begin{table}[ht]\caption{Accuracy (\%) of each model before and after applying ALIGN across three benchmark datasets.}
\label{tab:command}
\centering
\resizebox{.6\linewidth}{!}{%
\begin{tabular}{c|cc|cc|cc}
\toprule
\multirow{2}{*}{\textbf{Agent}} & \multicolumn{2}{c|}{\textbf{Math}} & \multicolumn{2}{c|}{\textbf{GSM8K}} & \multicolumn{2}{c}{\textbf{GSM-Hard}} \\
 & \textbf{Before} & \textbf{After} & \textbf{Before} & \textbf{After} & \textbf{Before} & \textbf{After} \\
\midrule
Mistral & 30.4 & 32.6 & 56.8 & 57.4 & 23.4 & 23.7 \\
Zephyr  & 28.8 & 29.4 &60.2  &60.1  & 19.3 & 19.5 \\
Phi  & 29.6& 30.6 &55.8  & 56.4  & 23.9 & 24.4 \\
\bottomrule
\end{tabular}}
\end{table}

\begin{table}[ht]\caption{Accuracy (\%) across benchmark datasets for the same agent. 
\textbf{Bold} indicates the best performance for each dataset.
FS-CoT: Few-shot CoT; SC: CoT-SC@maj16.}
\label{tab:accuracy_same_model}
\centering
\begin{tabular}{ccccc}
\toprule
Dataset &FS-CoT &SC & rStar & ALIGN \\
\midrule
Math &8.6 & 13.2&  33.0  &\textbf{34.4} \\
\bottomrule
\end{tabular}
\end{table}

\begin{table}[ht]
\caption{Accuracy (\%) of Mistral before and after applying ALIGN for MATH dataset.}
\label{tab:mistral_command}
\centering
\begin{tabular}{c|cc}
\toprule
\multirow{2}{*}{\textbf{Agent}} & \multicolumn{2}{c}{\textbf{Math}} \\
 & \textbf{Before} & \textbf{After} \\
\midrule
Mistra\_01(temp = 0.6) & 33.4 & 34.2 \\
Mistra\_02(temp = 0.8) & 30.4 & 31.4 \\
Mistra\_03(temp = 0.9) & 33.0 & 33.6 \\
\bottomrule
\end{tabular}
\end{table}

\begin{table}[t]
\centering
\caption{Accuracy (\%) across benchmark datasets for each method.
\label{tab:Llama_command}
\textbf{Bold} indicates the best performance for each dataset.}
\resizebox{.6\linewidth}{!}{%
\begin{tabular}{ccccc}
\toprule
Dataset & FS-CoT & Principal & rStar & ALIGN \\
\midrule
Math     & 8.0  & 4.8  & 26.8 & \textbf{29.1} \\
GSM8K    & 15.7 & 35.8 & 50.0 & \textbf{60.3} \\
GSM-Hard & 20.6 & 11.7 & 28.1 & \textbf{29.4} \\
\bottomrule
\end{tabular}}
\end{table}

\begin{table}[ht]
\centering
\caption{Accuracy (\%) of each model before and after applying ALIGN across three benchmark datasets. Bold highlights the best result for each dataset.}
\label{tab:LLAMA_command2}
\resizebox{.6\linewidth}{!}{%
\begin{tabular}{c|cc|c|cc|cc}
\toprule
\multirow{2}{*}{\textbf{Agent}} 
& \multicolumn{2}{c|}{\textbf{Math}} 
& \multirow{2}{*}{\textbf{Agent}} 
& \multicolumn{2}{c|}{\textbf{GSM8K}} 
& \multicolumn{2}{c}{\textbf{GSM-Hard}} \\
& \textbf{Before} & \textbf{After} 
& 
& \textbf{Before} & \textbf{After} 
& \textbf{Before} & \textbf{After} \\
\midrule
Mistral & 31.4 & 33.1 & Mistral & 39.0 & 39.3 & 27.2 & 27.5 \\
Zephyr  & 28.4 & 28.8 & Zephyr  & 60.7 & 61.0 & 22.2 & 21.9 \\
Falcon  & 18.7 & 16.1 & Phi     & 58.7 & 61.7 & 30.3 & 31.2 \\
\bottomrule
\end{tabular}}
\end{table}

\subsection{Validation for Assumption~\ref{assump:structure}}\label{sec:assump validation}
We empirically verify Assumption~\ref{assump:structure} in the experiments. Given the candidate answers generated by all the agents, we estimate the principal’s utility by submitting all agent's answers to the principal, collecting the feedback, and normalizing the scores.
For part (i), we denote a problem valid if no alternative answer yields strictly higher utility for both the principal and the agent than the submitted answers across all iterations. We then report the percentage of valid problems for each dataset and agent. As shown in Table~\ref{tab:validity}, this percentage is consistently around 90\% across agents and datasets, meaning that most of the problems satisfy the Pareto-optimal play in Assumption~\ref{assump:structure}.
For part (iii), we compute the Pearson correlation between the principal’s utilities and those of each agent for each dataset. Table~\ref{tab:correlation} reports the percentage of positively correlated problems and average correlation coefficients, which are positive for all datasets and agents. This result justifies the non-negative alignment hypothesis in Assumption~\ref{assump:structure}. Although the principal’s and agent’s objectives are not perfectly aligned, solving the same task induces partial alignment, so utilities move together on average rather than in opposite directions.

\begin{table}[ht]
\caption{Percentage of valid problems across all agents and datasets.}
\label{tab:validity}
\centering
% \resizebox{.6\linewidth}{!}{%
\begin{tabular}{c|ccc}
\toprule
\textbf{Agent} &\textbf{Math}& \textbf{GSM8K} & \textbf{GSM-Hard} \\
\midrule
Mistral & 90.4 & 91.0 & 97.6 \\
Zephyr  & 91.3 & 90.6 & 95.5 \\
Phi  & 89.2 & 91.4 & 96.0  \\
\bottomrule
\end{tabular}
\end{table}

\begin{table}[h!]
\caption{Verification of non-negative alignment between principal and agents. \emph{Positive} (\%) is the percentage of problems with positive correlation. \emph{Average} is the mean correlation across problems.}
\label{tab:correlation}
\centering
% \resizebox{.7\linewidth}{!}{%
\begin{tabular}{c|cc|cc|cc}
\toprule
\multirow{2}{*}{\textbf{Agent}} 
  & \multicolumn{2}{c|}{\textbf{Math}}
  & \multicolumn{2}{c|}{\textbf{GSM8K}} 
  & \multicolumn{2}{c}{\textbf{GSM-Hard}} \\
& \textbf{Positive} & \textbf{Average} 
& \textbf{Positive} & \textbf{Average} 
& \textbf{Positive} & \textbf{Average} \\
\midrule
Mistral & 91.3 & 0.4852 & 72.9 & 0.2611 & 82.7 & 0.4063 \\
Zephyr  & 91.3 & 0.5128 & 72.6 & 0.2672 & 82.9 & 0.3820 \\
Phi  & 92.5 & 0.4883 &66.2 & 0.1561 & 72.4 & 0.2204   \\
\bottomrule
\end{tabular}
\end{table}
%%%%%%%%%%%%%%%%%%%%%%%%%%%%%%%%%%%%%%%%%%%%%%%%%
\section{Related Works}
\label{sec:related works}
\textbf{LLM Diverse Reasoning.}
A growing body of work improves LLM performance at inference time, often referred to as \emph{test-time compute}. Popular approaches include (i) prompt-based techniques, such as chain-of-thought prompting to elicit structured reasoning \citep{wei2022chain}, and (ii) sampling and search, including Top-$k$, Top-$p$, beam search \citep{feng2023alphazero}, or tree-based exploration such as MCTS \citep{sutton1998reinforcement}. To further refine candidate outputs, methods such as majority voting \citep{wang2022self} and verifier models \citep{lightman2023let} have been employed to select high-quality responses.
A key insight is that sampling diverse reasoning paths (either entire trajectories or step-by-step expansions) significantly outperforms simply picking the most probable answers without exploration, in terms of accuracy and task completion rates \citep{snell2024scaling,brown2024large}.
Our method builds on this paradigm by framing inference as a game-theoretic process: multiple agents strategically generate diverse answers guided by designed prompts. Through reward design and competition in a repeated setting, the system incentivizes exploration of higher-quality reasoning paths while maintaining diversity.

\textbf{LLM Self-improvement.}
A growing line of work explores how LLMs improve reasoning through structured self-improvement without external supervision. Inspired by AlphaZero where learning emerges from play and feedback \citep{silver2017mastering}, LLMs can iteratively provide feedback, refinements, or critiques to improve its answers after generations \citep{madaan2023self, cheng2024self, chen2024self}. However, the effectiveness of this process often depends on the model’s inherent capabilities and may yield misleading gains in weaker models. Our approach aims to enable LLMs to self-improve by learning from feedback provided by a principal without relying on RL or fine-tuning.

\textbf{LLMs as Strategic Agents.}
With the advancement of LLMs, a growing body of research investigates their behavior in game-theoretic multi-agent settings. Empirical studies and benchmarks analyze LLM decision-making across diverse games, either collaborative or adversarial, with two or multiple agents, and under short- or long-term utility objectives \citep{lan2023llm,huang2024far,piatti2024cooperate}. Beyond evaluation, multiple LLMs have been organized into multi-agent systems that engage in debate, feedback exchange, or negotiation to improve answer quality \citep{huang2024ensemble,chen2025incentivizing,wang2024mobile}. Current work designs game-theoretic frameworks that directly enhance LLM reasoning and consistency. For example, adversarial games assign LLMs to attacker and defender roles \citep{cheng2024self,kirchner2024prover}, though these typically require reinforcement learning to train the policies and are often limited to two-agent settings. The consensus game framework in \citet{jacob2023consensus} offers a training-free approach that aligns generation and discrimination to promote consistency, but it also remains constrained to two agents and may converge to suboptimal equilibria. The work most closely aligned with ours is \citep{yi2025debate}, which also applies a game-theoretic framework to enhance multi-agent reasoning. However, our method specifically focuses on creating a competitive environment with ranking feedback to foster diverse reasoning and improvement without the need for fine-tuning or reinforcement learning. In contrast, their approach assumes prior Bayesian beliefs between agents, which is more aligned with agent coordination. 
% Our method extends this line of work by modeling multi-agent interactions with structured feedback to promote diverse reasoning and strategic improvement without fine-tuning or reinforcement learning.

%%%%%%%%%%%%%%%%%%%%%%%%%%%%%%%%%%%%%%%%%%%%%%%%%
\section{Conclusions}
\label{sec:conclusions}
In this work, we introduce ALIGN, a training-free, competitive multi-agent framework that improves LLM reasoning via aligned delegation. By modeling interactions between a principal and multiple agents as a game, ALIGN leverages competition and alignment incentives to elicit higher-quality outputs. We provide theoretical guarantees showing that, under fair comparisons, multi-agent systems outperform single-agent baselines. Using online mirror descent, each agent achieves sublinear regret, and the  average policies of agents converge to a Nash equilibrium. Empirical evaluations across diverse mathematical reasoning benchmarks demonstrated consistent improvements.

% While effective, our framework assumes reliable ranking feedback from the principal, which may limit applicability in more complex settings. 
While current experiments focus on math reasoning, future work will extend ALIGN to other reasoning tasks such as code generation and multi-step planning. Incorporating tools from game theory and mechanism design may further improve alignment and robustness among LLM agents.
%%%%%%%%%%%%%%%%%%%%%%%%%%%%%%%%%%%%%%%%%%%%%%%%%

\bibliography{aiAgent}
% \bibliographystyle{apalike}

%%%%%%%%%%%%%%%%%%%%%%%%%%%%%%%%%%%%%%%%%%%%%%%%%%%%%%%%%%%%%%%%%%%%%%%%%%%%%%%
%%%%%%%%%%%%%%%%%%%%%%%%%%%%%%%%%%%%%%%%%%%%%%%%%%%%%%%%%%%%%%%%%%%%%%%%%%%%%%%
% APPENDIX
%%%%%%%%%%%%%%%%%%%%%%%%%%%%%%%%%%%%%%%%%%%%%%%%%%%%%%%%%%%%%%%%%%%%%%%%%%%%%%%
%%%%%%%%%%%%%%%%%%%%%%%%%%%%%%%%%%%%%%%%%%%%%%%%%%%%%%%%%%%%%%%%%%%%%%%%%%%%%%%
\newpage
\appendix
\onecolumn

\section{Proofs}\label{sec: proofs}

\subsection{Proof of Theorem \ref{thm:multi-agent-delegation}}\label{appendix: proof 1}

\begin{lemma}[\cite{kleinberg2018delegated}]\label{lemma:single agent}
If \(M\) is any mechanism and \(\sigma\) is a best–response strategy (profile) to \(M\).
Let \(f_{M,\sigma}\) denote the interim allocation function, i.e.,
the function that specifies an outcome of \(M\) when
agents follow \(\sigma\). Then there exists a single proposal mechanism $M'$ and a best response $\sigma'$ to $M'$, 
such that the interim allocation functions $f_{M,\sigma}$ and $f_{M',\sigma'}$ are identical.
\end{lemma}

% \begin{proof}
% Let $R$ be the range of the interim allocation function $f_{M,\sigma}$, i.e., the set of all possible 
% outcomes of $M$, other than $\perp$, when the agent acts according to $\sigma$. 
% Define $M'$ to be the single-proposal mechanism with eligible set $R$. 
% Let $\sigma'$ be the strategy in which the agent observes his tuple of samples, 
% $\omega = (\omega_1, \ldots, \omega_n)$, and chooses strategy 
% $\sigma'(\omega) = g(\sigma(\omega))$. 
% By construction the interim allocation functions $f_{M,\sigma}$ and $f_{M',\sigma'}$ are identical. 
% To prove that $\sigma'$ is a best response to $M'$, consider any $\omega \in \Omega^n$ 
% and any $\nu \neq \sigma'(\omega)$. 
% Let $y_0 = U_y(g'(\sigma'(\omega)))$ denote the agent's utility when playing according to $\sigma'$; 
% note that $y_0 \ge 0$. 
% We wish to show that the agent cannot benefit by playing $\nu$ instead, i.e.,
% \begin{equation*}\label{eq:single agent}
%     U_y(g'(\nu)) \le y_0.
% \end{equation*}
% If $\nu \notin R$ then $g'(\nu) = \perp$ which implies Eq.~\eqref{eq:single agent} since $y_0 \ge 0$. 
% If $\nu \in R$ then $\nu = g(\tilde{\sigma})$ for some $\tilde{\sigma} \in \Sigma$. 
% Now Eq.~\eqref{eq:single agent} follows because strategy $\sigma$ is a best response for mechanism $M$, 
% and $U_y(g'(\nu))$ denotes the agent's utility when playing strategy $\tilde{\sigma}$ in $M$ 
% whereas $y_0$ denotes his utility when playing $\sigma$. 
% \end{proof}

\begin{proof}[Proof of Theorem \ref{thm:multi-agent-delegation} (part a)]
The proof directly follows Lemma~\ref{lemma:single agent} and Theorem E.3 in \citet{shin2023delegating}.
\end{proof}

For part b, We begin by defining a key event in which the answer chosen by an agent to maximize their own expected utility is also the one most preferred by the principal. 
Under the condition of Theorem \ref{thm:multi-agent-delegation}, the agent’s expected utility for proposing \(\omega\) is approximately
\[
\left[\frac{U(\omega)+1}{2}\right]^{k(N-1)} U_y(\omega),
\]
where \(k(N-1)\) is the total number of competing candidate answers from the other agents. This motivates the agent to select
\[
\argmax_{\omega \in \bar{\omega}_i} 
\{[U(\omega)+1]/2\}^{k(N-1)} U_y(\omega).
\]
However, the system designer would ideally prefer the agent to select
\(
\argmax_{\omega \in \bar{\omega}_i} U(\omega),
\)
which aligns with the principal's interest.
We therefore define the event \(E_i\) under which the agent’s optimal choice (based on their own objective) coincides with the principal’s preferred one:
\[
E_i(\{U(\omega), U_y(\omega)\}_{\omega \in \bar{\omega}_i}) 
= \left\{ \bar{\omega}_i : \argmax_{\omega \in \bar{\omega}_i} \{[U(\omega)+1]/2\}^{k(N-1)} U_y(\omega) 
= \argmax_{\omega \in \bar{\omega}_i} U(\omega) \right\}.
\]
In other words, this event captures the favorable case when the agent’s self-interested decision also maximizes the principal’s utility.
For notational simplicity, we refer to \(E_i(\{U(\omega), U_y(\omega)\}_{\omega \in \bar{\omega}_i})\) simply as \(E_i\) or \(E_i(\bar{\omega}_i)\).
At the joint level, define the event:
\[
E(\bar{\omega}) := E(U(\bar{\omega}), U_y(\bar{\omega})) 
= \bigcap_{i \in [N]} E_i(\bar{\omega}_i),
\]
which holds if all agents simultaneously select answers that align with the principal's preference.

Now, we will prove that given $E(U(\bar{\omega}), U_y(\bar{\omega}))$ and the other’s strategies, 
proposing a answer that maximizes $U(\cdot)$ will be approximately best-response. 

% Let’s restrict our attention to the case in which $E$ happens given $\bar{\omega}$. 
% Suppose that agent $j \neq i$ proposes a candidate to maximize $U(\cdot)$, 
% i.e., $\sigma_j = \argmax_{\omega \in \bar{\omega}_j} U(\omega)$. 
% Then, conditioned on $E$, agent $i$’s expected utility of proposing $\omega \in \bar{\omega}_i$ can be obtained as
% \begin{align*}
% U_i(\omega, \sigma_{-i} \mid E(\bar{\omega})) 
% &= \mathbb{P}\,[\omega \text{ is winner} \mid E(\bar{\omega})] \cdot U_y(\omega)\\
% &= \mathbb{P}\,[U(\omega) \geq U(\omega'), \, \forall \omega' \in \bar{\omega}_{-i} \mid E(\bar{\omega})] \cdot U_y(\omega) \\&
% = U(\omega)^{k(N-1)} U_y(\omega).
% \end{align*}

% Since we conditioned on $E(\bar{\omega})$, for agent $i$, the answer that maximizes 
% $U(\omega)^{k(N-1)} U_y(\omega)$ is exactly the same as the answer that maximizes $U(\omega)$. 
% This verifies that conditioned on $E(\bar{\omega})$ and assuming that the other agent $j \neq i$ 
% proposes an answer that maximizes $U(\cdot)$, also maximizing $U(\cdot)$ is agent $i$’s best response. 
Let $\sigma_i^x$ be a strategy to propose a answer that maximizes $U(\cdot)$ for agent $i$.
Given that all the other agents $j$ playing $\sigma_j^x$ for $j \neq i$, 
we can further obtain the following regarding agent $i$’s utility:
\begin{align*}
U_i(\sigma_i^x, \sigma_{-i}) 
&= \mathbb{E}\,[U_i(\sigma_i^x, \sigma_{-i}) \mid E] \, \mathbb{P}[E] 
   + \mathbb{E}\,[U_i(\sigma_i^x, \sigma_{-i}) \mid E^c] \, \mathbb{P}[E^c] \\
&\geq \mathbb{E}\,[U_i(\sigma_i^x, \sigma_{-i}) \mid E] \, \mathbb{P}[E] - \mathbb{P}[E^c]  \\
&\geq \mathbb{E}\,[U_i(\sigma_i', \sigma_{-i}) \mid E] \, \mathbb{P}[E] - \mathbb{P}[E^c]\\
&= \mathbb{E}\,[U_i(\sigma_i', \sigma_{-i})] - \mathbb{E}\,[U_i(\sigma_i', \sigma_{-i} \mid E^c)] \, \mathbb{P}[E^c] - \mathbb{P}[E^c]\\
&\geq \mathbb{E}\,[U_i(\sigma_i', \sigma_{-i})] - 2\mathbb{P}[E^c],
\end{align*}
where the second inequality follows from the fact that given $E$, 
playing $\sigma_i^x$ is weakly dominant over any other strategy $\sigma_i'$ for agent $i$. 

This implies that if we characterize a good lower bound $\alpha$ such that $\mathbb{P}[E] \geq \alpha$, we have
\[
U_i(\sigma_i^x, \sigma_{-i}) \;\geq\; U_i(\sigma_i', \sigma_{-i}) - 2(1 - \alpha),
\]
which implies that $\sigma_i^x$ is $2(1-\alpha)$-approximate BNE. The next step will be to consider the lower bound of event \(E\).

\begin{lemma}[\cite{shin2023delegating}]\label{lemma:independent bne} 
Under Assumption \ref{assump:structure}, assume the utility function of principals and agents are independent, we have \(\mathbb{P}[E]\geq \alpha_0 = e^{-\frac{N^2}{2(N-1)^2}}\).
\end{lemma}

\begin{lemma}\label{lemma:correlation}
Suppose \(U(\omega)\) and \(U_y(\omega)\) are positively correlated. Then the probability of the event \(E_i(\bar{\omega}_i) \)
is weakly greater than in the case where \(U(\omega)\) and \(U_y(\omega)\) are independent:
\[
\mathbb{P}_{U \perp U_y}[E_i(\bar{\omega}_i)] \leq \mathbb{P}_{U \uparrow U_y}[E_i(\bar{\omega}_i)].
\]
\end{lemma}

\begin{proof}
We prove the lemma using a coupling argument based on the stochastic dominance induced by positive correlation.

Let \(\bar{\omega}_i\) be the set of proposals for agent \(i\), with \(|\bar{\omega}_i| = k\). Let \(\omega^* = \operatorname{argmax}_{\omega \in \bar{\omega}_i} U(\omega)\) be the proposal that maximizes \(U\), and let \(x_1 = U(\omega^*)\). For \(j = 2, \ldots, k\), let \(x_j\) be the \(x\)-values of the other proposals, ordered such that \(x_1 \geq x_2 \geq \cdots \geq x_k\). Define \(B_j = \left( \frac{x_j}{x_1} \right)^{k(N-1)}\) for \(j = 2, \ldots, k\).

The event \(E_i\) occurs if and only if for all \(j = 2, \ldots, k\),
\[
\frac{U_y(\omega^*)}{U_y(\omega_j)} \geq B_j.
\]
Define \(C_j = \left\{ \frac{U_y(\omega^*)}{U_y(\omega_j)} \geq B_j \right\}\), so that \(E_i = \cap_{j=2}^k C_j\).

Now, condition on the \(x\)-values \(\bm{x} = (x_1, \ldots, x_k)\). In the independent case, \(U_y(\omega)\) is drawn from the marginal distribution \(F_y\) for each \(\omega\), independently. In the positively correlated case, \(U_y(\omega)\) is drawn from the conditional distribution \(F_{y|x(\omega)}\). By positive correlation, for any \(x_1 > x_j\), the distribution \(F_{y|x_1}\) stochastically dominates \(F_{y|x_j}\), i.e.,
\[
\mathbb{P}[y > t \mid x_1] \geq \mathbb{P}[y > t \mid x_j] \quad \forall t.
\]
In particular, \(F_{y|x_1}\) stochastically dominates the marginal distribution \(F_y\), and \(F_y\) stochastically dominates \(F_{y|x_j}\) for \(j \geq 2\).

By stochastic dominance, we can construct coupled random variables as follows:
\begin{itemize}
    \item For \(U_y(\omega^*)\), let \(Y_1^+ \sim F_{y|x_1}\) and \(Y_1^* \sim F_y\) such that \(Y_1^+ \geq Y_1^*\) almost surely.
    \item For each \(U_y(\omega_j)\) with \(j \geq 2\), let \(Y_j^- \sim F_{y|x_j}\) and \(Y_j^* \sim F_y\) such that \(Y_j^- \leq Y_j^*\) almost surely.
\end{itemize}
Such couplings exist due to the stochastic dominance relations.

Now, for each \(j \geq 2\), almost surely,
\[
\frac{Y_1^+}{Y_j^-} \geq \frac{Y_1^*}{Y_j^*}.
\]
Therefore,
\[
\left\{ \frac{Y_1^*}{Y_j^*} \geq B_j \right\} \subseteq \left\{ \frac{Y_1^+}{Y_j^-} \geq B_j \right\}.
\]
This implies that for each \(j\),
\[
\mathbb{P}\left[ \frac{Y_1^+}{Y_j^-} \geq B_j \right] \geq \mathbb{P}\left[ \frac{Y_1^*}{Y_j^*} \geq B_j \right].
\]

Since the \(y\)-values are conditionally independent given \(\bm{x}\), the events \(C_j\) are conditionally independent given \(y(\omega^*)\) in both cases. However, by the above coupling, we have almost surely,
\[
\bigcap_{j=2}^k \left\{ \frac{Y_1^*}{Y_j^*} \geq B_j \right\} \subseteq \bigcap_{j=2}^k \left\{ \frac{Y_1^+}{Y_j^-} \geq B_j \right\}.
\]
Thus, the joint probability satisfies:
\[
\mathbb{P}\left[ \bigcap_{j=2}^k \left\{ \frac{Y_1^+}{Y_j^-} \geq B_j \right\} \mid \bm{x} \right] \geq \mathbb{P}\left[ \bigcap_{j=2}^k \left\{ \frac{Y_1^*}{Y_j^*} \geq B_j \right\} \mid \bm{x} \right].
\]
The right-hand side is the conditional probability of \(E_i\) in the independent case, and the left-hand side is the conditional probability in the positively correlated case.

Taking expectation over \(\bm{x}\), we obtain:
\[
\mathbb{P}_{U \uparrow U_y}[E_i(\bar{\omega}_i)] \geq \mathbb{P}_{U \perp U_y}[E_i(\bar{\omega}_i)],
\]
which completes the proof.
\end{proof}

By Lemma~\ref{lemma:independent bne}, when \(U(\omega)\) and \(U_y(\omega)\) are independent and identically distributed across candidates, the probability of the alignment event \(E\) admits an explicit lower bound. This implies the existence of a \(2(1 - \alpha)\)-approximate BNE under the independent setting, where \(\alpha = 1 - \mathbb{P}[E]\).
By Lemma~\ref{lemma:correlation}, when \(U(\omega)\) and \(U_y(\omega)\) are positively correlated, the probability of \(E\) is weakly greater than in the independent case.
Therefore, the same approximation bound holds in the positively correlated setting.

\subsection{Proof of Theorem \ref{thm:no regret}}

We begin with some auxiliary lemmas that characterize the structure of the entropy-regularized policy updates. Specifically, we leverage the optimality conditions induced by mirror descent with negative entropy to relate policy differences to KL divergence terms. These lemmas follow standard arguments in the online learning literature, but we include them here for completeness and to establish notation for our subsequent regret analysis \citep{jacob2022modeling,bakhtin2022mastering,shalev2012online}. 

\begin{lemma}
At any round \( t \), if player \( i \)'s policy update follows:
\[
\pi_i^{t+1} = \argmax_{\pi} \left\{ \sum_{t'=1}^{t} U_i^{t'}(\pi) - \frac{1}{\eta} \varphi(\pi) \right\},
\]
where the regularizer is the negative entropy,
\[
\varphi(\pi) := \sum_{a \in \mathcal{A}_i} \pi(a) \log \pi(a),
\]
then the resulting policy \(\pi_i^{t+1}\) is equivalent to the one generated by Algorithm~\ref{algorithm}.
\end{lemma}

\begin{proof}
We begin by rewriting the cumulative utility term as a linear function over actions:
\[
\sum_{t'=1}^{t} U_i^{t'}(\pi) = \sum_{a \in \mathcal{A}_i} \left( \sum_{t'=1}^{t} U_i(a, a_{-i}^{t'}) \right) \pi(a).
\]
Substituting into the objective, the optimization becomes:
\[
\max_{\pi} \left\{ \sum_{a \in \mathcal{A}_i} \left( \sum_{t'=1}^{t} U_i(a, a_{-i}^{t'}) \right) \pi(a) - \frac{1}{\eta} \sum_{a \in \mathcal{A}_i} \pi(a) \log \pi(a) \right\}.
\]
This is a standard instance of entropy-regularized maximization over a simplex. Its answer is well known to be a softmax distribution over accumulated utilities:
\[
\pi_i^{t+1}(a) = \frac{\exp\left( \eta \sum_{t'=1}^{t} U_i(a, a_{-i}^{t'}) \right)}{\sum_{a' \in A_i} \exp\left( \eta \sum_{t'=1}^{t} U_i(a', a_{-i}^{t'}) \right)}.
\]
This precisely matches the update rule employed in Algorithm~\ref{algorithm}, where action values are incrementally aggregated and exponentiated with temperature \( \eta \). Thus, the two procedures are equivalent.
\end{proof}

\begin{lemma}
Suppose that at each round \( t+1 \), player \( i \) updates their policy via the following optimization:
\[
\pi_i^{t+1} = \argmax_{\pi} \left\{ \sum_{t'=1}^{t} \tilde{U}_i(\pi, a_{-i}^{t'}) - \frac{1}{\eta} \varphi(\pi) \right\},
\]
where \( \eta > 0 \) is a fixed parameter, the regularizer \( \varphi \) is the negative Shannon entropy
\[
\varphi(\pi) := \sum_{a \in \mathcal{A}_i} \pi(a) \log \pi(a),
\]
and the shifted utility function is defined as
\[
\tilde{U}_i(a, a_{-i}^{t}) := U_i(a, a_{-i}^t) - \min_{\mathbf{a} \in \mathcal{A}_1 \times \cdots \times \mathcal{A}_N} U_i(\mathbf{a}).
\]
Then, the resulting policy admits a softmax representation:
\[
\pi_i^{t+1}(a) = \frac{\exp\left(v_i^{t+1}(a)\right)}{\sum_{a' \in \mathcal{A}_i} \exp\left(v_i^{t+1}(a')\right)} \quad \forall a \in \mathcal{A}_i,
\]
where
\[
v_i^{t+1}(a) := \eta \sum_{t'=1}^{t} \tilde{U}_i(a, a_{-i}^{t'}).
\]
\end{lemma}

\begin{proof}
To simplify notation, let \( \gamma := \min_{\mathbf{a} \in \mathcal{A}_1 \times \cdots \times \mathcal{A}_N} \tilde{U}_i(\mathbf{a}) \). Since the minimum utility is constant across all actions, subtracting it from the original utilities does not affect the maximizer of the objective. Thus, the cumulative utility component can be rewritten as:
\[
\sum_{t'=1}^t \tilde{U}_i(\pi, a_{-i}^{t'}) = \sum_{a \in \mathcal{A}_i} \left( \sum_{t'=1}^t \tilde{U}_i(a, a_{-i}^{t'}) \right) \pi(a).
\]
Substituting into the objective function yields:
\[
\pi_i^{t+1} = \argmax_{\pi} \left\{ \eta \sum_{a \in \mathcal{A}_i} \left( \sum_{t'=1}^{t} \tilde{U}_i(a, a_{-i}^{t'}) \right) \pi(a) - \sum_{a \in \mathcal{A}_i} \pi(a) \log \pi(a) \right\}.
\]
This is a classical instance of entropy-regularized linear optimization over the probability simplex. The optimal answer is known to be a softmax distribution over the accumulated (shifted) utility scores:
\[
\pi_i^{t+1}(a) = \frac{\exp\left(v_i^{t+1}(a)\right)}{\sum_{a' \in A_i} \exp\left(v_i^{t+1}(a')\right)} \quad \forall a \in \mathcal{A}_i,
\]
where \( v_i^{t+1}(a) := \eta \sum_{t'=1}^{t} \tilde{U}_i(a, a_{-i}^{t'}) \), completing the proof.
\end{proof}

\begin{lemma}\label{lemma:iteration zero}
Let \( t \geq 1 \) and fix agent \( i \). Suppose \( \pi_i^t \) and \( \pi_i^{t+1}  \) are the policy updates produced by Algorithm~\ref{algorithm} at iteration \(t\) and \(t+1\) respectively. Then for any pair of distributions \( \pi, \pi'\), the following identity holds:
\[
\left\langle -\eta U_i^t + \nabla \varphi(\pi_i^{t+1}) - \nabla \varphi(\pi_i^t), \pi - \pi' \right\rangle = 0.
\]
\end{lemma}

\begin{proof}
We analyze the policy update dynamics via the optimality conditions associated with the mirror descent updates under negative entropy regularization. Define the empirical utility vectors at rounds \( t-1 \) and \( t \) as
\[
\bar{U}_i^{t-1} := \frac{1}{t-1} \sum_{s=1}^{t-1} u_i^s, \qquad \bar{U}_i^t := \frac{1}{t} \sum_{s=1}^{t} u_i^s.
\]
From the KKT conditions of the optimization problems defining \( \pi_i^{t} \) and \( \pi_i^{t+1} \), we know that:
\[
\left\langle -\bar{U}_i^t + \frac{1}{\eta t} \nabla \varphi(\pi_i^{t+1}), \pi - \pi' \right\rangle = 0, \quad 
\left\langle -\bar{U}_i^{t-1} + \frac{1}{\eta (t-1)} \nabla \varphi(\pi_i^t), \pi - \pi' \right\rangle = 0.
\]
Subtracting the second equation from the first gives:
\begin{align*}
0 = \left\langle 
\bar{U}_i^{t-1} - \bar{U}_i^t 
+ \frac{1}{\eta t} \nabla \varphi(\pi_i^{t+1}) 
- \frac{1}{\eta(t-1)} \nabla \varphi(\pi_i^t), \; 
\pi - \pi' \right\rangle.
\end{align*}

Next, observe that:
\[
\bar{U}_i^t = \frac{t-1}{t} \bar{U}_i^{t-1} + \frac{1}{t} U_i^t
\quad \Rightarrow \quad
\bar{U}_i^t - \bar{U}_i^{t-1} = -\frac{1}{t} \bar{U}_i^{t-1} + \frac{1}{t} U_i^t,
\quad \text{or} \quad 
\bar{U}_i^{t-1} - \bar{U}_i^t = \frac{1}{t} (\bar{U}_i^{t-1} - U_i^t).
\]
Substituting this into the previous expression yields:
\[
0 = \left\langle 
\frac{1}{t} (\bar{U}_i^{t-1} - u_i^t)
+ \frac{1}{\eta t} \nabla \varphi(\pi_i^{t+1}) 
- \frac{1}{\eta(t-1)} \nabla \varphi(\pi_i^t),
\; \pi - \pi' \right\rangle.
\]
Rewriting \( \bar{U}_i^{t-1} = \bar{U}_i^t - \frac{1}{t}(u_i^t - \bar{U}_i^t) \) and simplifying, we reach:
\[
0 = \left\langle 
\frac{-1}{t-1} U_i^t 
+ \frac{1}{\eta(t-1)} \nabla \varphi(\pi_i^{t+1}) 
- \frac{1}{\eta(t-1)} \nabla \varphi(\pi_i^t),
\; \pi - \pi' \right\rangle.
\]
Multiplying through by \( \eta(t - 1) \) gives the desired result:
\[
\left\langle -\eta U_i^t + \nabla \varphi(\pi_i^{t+1}) - \nabla \varphi(\pi_i^t), \pi - \pi' \right\rangle = 0.
\]
\end{proof}

\begin{lemma}
Let \( i \) be any agent and \( t \geq 1 \). Suppose that the policies \( \pi_i^t \) and \( \pi_i^{t+1} \) are the successive updates obtained from Algorithm~\ref{algorithm}. Then, for any policy \( \pi  \), the following identity holds:
\[
\left\langle -U_i^t,\; \pi - \pi_i^{t+1} \right\rangle
= \frac{1}{\eta} \left( D_{\mathrm{KL}}(\pi \| \pi_i^t) - D_{\mathrm{KL}}(\pi \| \pi_i^{t+1}) + D_{\mathrm{KL}}(\pi_i^{t+1} \| \pi_i^t) \right).
\]
\end{lemma}

\begin{proof}
We begin by recalling from Lemma~\ref{lemma:iteration zero} that the following condition is satisfied by the update rule for any pair \( \pi, \pi' \):
\[
\left\langle -\eta U_i^t + \nabla \varphi(\pi_i^{t+1}) - \nabla \varphi(\pi_i^t),\; \pi - \pi' \right\rangle = 0.
\]
By choosing \( \pi' = \pi_i^{t+1} \), this reduces to:
\[
\left\langle -\eta U_i^t + \nabla \varphi(\pi_i^{t+1}) - \nabla \varphi(\pi_i^t),\; \pi - \pi_i^{t+1} \right\rangle = 0.
\]
Rearranging terms gives:
\[
\eta \left\langle -U_i^t, \pi - \pi_i^{t+1} \right\rangle
= \left\langle \nabla \varphi(\pi_i^t) - \nabla \varphi(\pi_i^{t+1}),\; \pi - \pi_i^{t+1} \right\rangle.
\]
We now invoke the identity for the Bregman divergence associated with the negative entropy function \( \varphi \), which yields:
\[
\left\langle \nabla \varphi(\pi_i^t) - \nabla \varphi(\pi_i^{t+1}),\; \pi - \pi_i^{t+1} \right\rangle
= D_{\mathrm{KL}}(\pi \| \pi_i^t) - D_{\mathrm{KL}}(\pi \| \pi_i^{t+1}) + D_{\mathrm{KL}}(\pi_i^{t+1} \| \pi_i^t).
\]

Putting everything together, we obtain:
\[
\eta \left\langle -U_i^t, \pi - \pi_i^{t+1} \right\rangle
= D_{\mathrm{KL}}(\pi \| \pi_i^t) - D_{\mathrm{KL}}(\pi \| \pi_i^{t+1}) + D_{\mathrm{KL}}(\pi_i^{t+1} \| \pi_i^t).
\]

Dividing both sides by \( \eta \) completes the proof.
\end{proof}

\begin{lemma}\label{lemma:bound}
For any agent \( i \) and round \( t \), the following upper bound holds for all policies \( \pi \):
\[
U_i^t(\pi) - U_i^t(\pi_i^t)
\leq \frac{\eta \|U_i^t\|_\infty^2}{4}
- \frac{1}{\eta} D_{\mathrm{KL}}(\pi \| \pi_i^{t+1})
+ \frac{1}{\eta} D_{\mathrm{KL}}(\pi \| \pi_i^t).
\]
\end{lemma}

\begin{proof}
We begin by illustrating Lemma~\ref{lemma:iteration zero}, which characterizes the optimality condition at each step via:
\[
\left\langle -\eta U_i^t + \nabla \varphi(\pi_i^{t+1}) - \nabla \varphi(\pi_i^t),\; \pi - \pi_i^{t+1} \right\rangle = 0 \quad \text{for all } \pi.
\]
Rearranging gives:
\[
\left\langle U_i^t, \pi - \pi_i^{t+1} \right\rangle
= \frac{1}{\eta} \left\langle \nabla \varphi(\pi_i^{t+1}) - \nabla \varphi(\pi_i^t),\; \pi - \pi_i^{t+1} \right\rangle.
\]
We now apply the three-point identity for Bregman divergence induced by the negative entropy regularizer:
\[
\left\langle \nabla \varphi(\pi_i^{t+1}) - \nabla \varphi(\pi_i^t),\; \pi - \pi_i^{t+1} \right\rangle
= D_{\mathrm{KL}}(\pi \| \pi_i^t) - D_{\mathrm{KL}}(\pi \| \pi_i^{t+1}) + D_{\mathrm{KL}}(\pi_i^{t+1} \| \pi_i^t).
\]
Substituting this gives:
\[
\left\langle U_i^t,\; \pi - \pi_i^{t+1} \right\rangle
= \frac{1}{\eta} \left( D_{\mathrm{KL}}(\pi \| \pi_i^t) - D_{\mathrm{KL}}(\pi \| \pi_i^{t+1}) + D_{\mathrm{KL}}(\pi_i^{t+1} \| \pi_i^t) \right).
\]
To relate this to \( U_i^t(\pi) - U_i^t(\pi_i^t) \), we subtract and add \( U_i^t(\pi_i^t) \), and observe:
\[
U_i^t(\pi) - U_i^t(\pi_i^t)
= \left\langle U_i^t,\; \pi - \pi_i^{t+1} \right\rangle + \left\langle U_i^t,\; \pi_i^{t+1} - \pi_i^t \right\rangle.
\]
Combining with the expression above, we have:
\begin{align*}
U_i^t(\pi) - U_i^t(\pi_i^t)
&= \frac{1}{\eta} \left( D_{\mathrm{KL}}(\pi \| \pi_i^t) - D_{\mathrm{KL}}(\pi \| \pi_i^{t+1}) + D_{\mathrm{KL}}(\pi_i^{t+1} \| \pi_i^t) \right) + \left\langle U_i^t,\; \pi_i^{t+1} - \pi_i^t \right\rangle.
\end{align*}
Next, apply Young’s inequality to the final term:
\[
\left\langle U_i^t, \pi_i^{t+1} - \pi_i^t \right\rangle
\leq \frac{\eta}{4} \|U_i^t\|_\infty^2 + \frac{1}{\eta} \|\pi_i^{t+1} - \pi_i^t\|_1^2.
\]
Finally, use the strong convexity of KL divergence to bound:
\[
\|\pi_i^{t+1} - \pi_i^t\|_1^2 \leq 2 D_{\mathrm{KL}}(\pi_i^{t+1} \| \pi_i^t),
\]
and thus:
\[
\left\langle U_i^t,\; \pi_i^{t+1} - \pi_i^t \right\rangle
\leq \frac{\eta}{4} \|U_i^t\|_\infty^2 + \frac{2}{\eta} D_{\mathrm{KL}}(\pi_i^{t+1} \| \pi_i^t).
\]
Putting all terms together, we arrive at:
\[
U_i^t(\pi) - U_i^t(\pi_i^t)
\leq \frac{\eta}{4} \|U_i^t\|_\infty^2 + \frac{1}{\eta} D_{\mathrm{KL}}(\pi \| \pi_i^t) - \frac{1}{\eta} D_{\mathrm{KL}}(\pi \| \pi_i^{t+1}),
\]
as desired.
\end{proof}

\begin{proof}[Proof of Theorem~\ref{thm:no regret}]
Summing Lemma~\ref{lemma:bound} over \( t = 0 \) to \( T \), the KL divergence terms telescope. Since \( D_{\mathrm{KL}}(\pi \| \pi_i^{T+1}) \geq 0 \), we have:
\begin{align*}
\sum_{t=1}^T U_i^t(\pi) - U_i^t(\pi_i^t)
&\leq \frac{\eta}{4} \sum_{t=1}^T \|U_i^t\|_\infty^2 + \frac{1}{\eta} D_{\mathrm{KL}}(\pi \| \pi_i^0) \\
&\leq \frac{\eta T}{4} + \frac{D_{\mathrm{KL}}(\pi \| \pi_i^0)}{\eta},
\end{align*}
where the last inequality uses \( \|U_i^t\|_\infty \leq 1 \). Setting \(\eta=\frac{1}{\sqrt{T}}\) completes the proof.
\end{proof}

\subsection{Proof of Theorem \ref{thm:convergence}}
In this section, we present the proof of Theorem \ref{thm:convergence}.

\begin{proof}[Proof of Theorem \ref{thm:convergence}]
Fix an agent $i \in [N]$, and any policy $\pi^* $, and introduce the discrete-time stochastic process
\[
w^t := \left( U_i(\pi^*, \pi_{-i}^t) - U_i(\pi_i^t, \pi_{-i}^t) \right) - \left( U_i(\pi^*, a_{-i}^t) - U_i(\pi_i^t, a_{-i}^t) \right).
\]
Since each opponent player $j \ne i$ plays according to Algorithm~\ref{algorithm}, the answers $a_{-i}^t$ at each round $t$ is sampled from the joint policy $\pi_{-i}^t$. Therefore, $w^t$ is a martingale difference sequence. Furthermore, by expanding the definition of $U_i$, the absolute value of $w^t$ satisfies
\begin{align*}
|w^t| 
&= \left| \left( U_i(\pi^*, \pi_{-i}^t) - U_i(\pi_i^t, \pi_{-i}^t) \right) - \left( U_i(\pi^*, a_{-i}^t) - U_i(\pi_i^t, a_{-i}^t) \right) \right| \\
&\leq \left| U_i(\pi^*, \pi_{-i}^t) - U_i(\pi^*, a_{-i}^t) \right| - \left| U_i(\pi_i^t, \pi_{-i}^t) - U_i(\pi_i^t, a_{-i}^t) \right| \\
&\leq 2.
\end{align*}

Hence, using Azuma-Hoeffding’s inequality, for any $\delta \in (0,1)$,
\begin{align*}
1 - \delta 
&\leq \mathbb{P} \left[ \sum_{t=1}^T w^t \leq  \sqrt{8T \log \frac{1}{\delta}} \right] \\
&= \mathbb{P} \left[ 
\left( \sum_{t=1}^T U_i(\pi^*, \pi_{-i}^t) - \sum_{t=1}^T U_i(\pi_i^t, \pi_{-i}^t) \right)
-
\left( \sum_{t=1}^T U_i(\pi^*, a_{-i}^t) - \sum_{t=1}^T U_i(\pi_i^t, a_{-i}^t) \right)
\leq \sqrt{8T \log \frac{1}{\delta}} \right] \\
&= \mathbb{P} \left[ 
\sum_{t=1}^T U_i(\pi^*, \pi_{-i}^t) - \sum_{t=1}^T U_i(\pi_i^t, \pi_{-i}^t)
\leq R_i^T + \sqrt{8T \log \frac{1}{\delta}} \right],
\end{align*}
Since the above expression holds for any $\pi^*$, in particular, using the union bound,
\[
\mathbb{P} \left[
\max_{\pi^* } \sum_{t=1}^T U_i(\pi^*, \pi_{-i}^t) - \sum_{t=1}^T U_i(\pi_i^t, \pi_{-i}^t)
\leq R_i^T + \sqrt{8T \log \frac{|\mathcal{A}_i|}{\delta}}
\right] \geq 1 - \delta.
\]
Summing for $i \in \{1, \dots, N\}$ and using the union bound, we can further write
\begin{align*}
\mathbb{P} \left[
\sum_{i=1}^N \max_{\pi_i^* } \left\{ \sum_{t=1}^T U_i(\pi_i^*, \pi_{-i}^t) \right\}
- \sum_{t=1}^T \sum_{i=1}^N U_i(\pi_1^t, \dots, \pi_N^t)
\leq \sum_{i=1}^N R_i^T + N\sqrt{8T \log \frac{\max_i |\mathcal{A}_i|}{\delta}} 
\right] \geq 1 - N\delta.
\end{align*}

Dividing by $T$ and noting that for any player $i \in \{1, \dots, N\}$,
\[
\frac{1}{T} \sum_{t=1}^T U_i(\pi_i^*, \pi_{-i}^t) 
= U_i\left( \pi_i^*, \frac{1}{T} \sum_{t=1}^T \pi_{-i}^t \right)
= U_i\left( \pi_i^*, \bar{\pi}_{-i}^T \right),
\]
further yields
\begin{align*}
\mathbb{P} \left[
\sum_{i=1}^N \max_{\pi_i^*} U_i(\pi_i^*, \bar{\pi}_{-i}^T) 
- \frac{1}{T} \sum_{t=1}^T \sum_{i=1}^N U_i(\pi_1^t, \dots, \pi_N^t)
\leq \sum_{i=1}^N \frac{R_i^T}{T} + N\sqrt{ \frac{8}{T} \log \frac{\max_i |\mathcal{A}_i|}{\delta} }
\right] \geq 1 - N\delta.
\end{align*}

We now analyze the term
\(
\Delta := -\frac{1}{T} \left( \sum_{t=1}^T \sum_{i=1}^N U_i(\pi_1^t, \dots, \pi_N^t) \right),
\)
which can be expressed as
\begin{align*}
\Delta 
= -\frac{1}{T} \sum_{t=1}^T \sum_{i=1}^N U_i(\pi_1^t, \dots, \pi_N^t)
= -\frac{1}{T} \sum_{t=1}^T \sum_{i=1}^N U_i(\pi_1^t, \dots, \pi_N^t) 
= - \sum_{i=1}^N U_i\left( \bar{\pi}_1^T, \dots, \bar{\pi}_N^T \right).
\end{align*}
% where $\bar{\pi}_i^T = \frac{1}{T} \sum_{t=1}^T \pi_i^t$ denotes the average strategy of player $i$.
Therefore we have
\[
\mathbb{P} \left[
\sum_{i=1}^N \max_{\pi_i^*}
\left\{ U_i(\pi_i^*, \bar{\pi}_{-i}^T) - U_i(\bar{\pi}_i^T, \bar{\pi}_{-i}^T) \right\}
\leq \sum_{i=1}^N \frac{R_i^T}{T} + N\sqrt{ \frac{8}{T} \log \frac{ \max_i |\mathcal{A}_i| }{ \delta } }
\right] \geq 1 - N\delta.
\]

Since
\(
\max_{\pi_i^*} \left\{ U_i(\pi_i^*, \bar{\pi}_{-i}^T) - U_i(\bar{\pi}_i^T, \bar{\pi}_{-i}^T) \right\} \geq 0\) for all \(i\),
the inequality above  implies that
\[
\mathbb{P} \left[
\max_{i \in [N]} \max_{\pi_i^* }
\left\{ U_i(\pi_i^*, \bar{\pi}_{-i}^T) - U_i(\bar{\pi}_i^T, \bar{\pi}_{-i}^T) \right\}
\leq \sum_{i=1}^N \frac{R_i^T}{T} + N\sqrt{ \frac{8}{T} \log \frac{ \max_i |\mathcal{A}_i| }{ \delta } }
\right] \geq 1 - N\delta.
\]

which is equivalent to the statement after making the variable substitution $\delta := \delta'/N$.
\end{proof}

\subsection{Necessity of Assumption \ref{assump:structure}}
\label{appendix:necessity assumptions}

Recall our definition that a mechanism $M$, under agent strategies $\sigma$, is said to be $(\rho, \gamma)$-approximate if its expected outcome satisfies:
\[
\rho\mathbb{E}[U_{M,\sigma}] + \gamma \;\geq\; \mathbb{E}[U_{\max}],
\]
where $\rho$ and $\gamma$ represent the multiplicative and additive approximation factors respectively. Following \citet{shin2023delegating}, we define the \emph{price of anarchy} (PoA) as a measure of the worst-case efficiency loss due to strategic behavior. Specifically, the multiplicative price of anarchy $\text{PoA}_m$ is the smallest $\rho$ such that the mechanism is $(\rho, 0)$-approximate under every Nash equilibrium. Similarly, the additive price of anarchy $\text{PoA}_a$ is the smallest $\gamma$ such that the mechanism is $(1, \gamma)$-approximate under all equilibria. 
In contrast, the \emph{price of stability} (PoS) captures the best-case performance at equilibrium: the multiplicative price of stability $\text{PoS}_m$ is the smallest $\rho$ such that there exists some Nash equilibrium under which the mechanism is $(\rho, 0)$-approximate, and the additive price of stability $\text{PoS}_a$ is similarly defined for $(1, \gamma)$-approximation. 
\begin{lemma}[\cite{shin2023delegating}]\label{lemma:symmetric}
Suppose that both the principal’s utility functions are supported on $[0, L]$ for some $L > 0$. For any $\varepsilon > 0$, there exists a problem instance such that $\text{PoS}_a \geq L - \varepsilon$, i.e. \(\mathbb{E}[U_{M,\sigma}]\leq \varepsilon\).
\end{lemma}

% \begin{lemma}
% If $R_i \neq \Omega$ for some $i \in [n]$, then the additive price of anarchy can be arbitrarily close to $L$, i.e., for any $\varepsilon > 0$ there exists a problem instance such that $\mathrm{PoA}_a \ge L - \varepsilon$.
% \end{lemma}

\begin{lemma}[\cite{shin2023delegating}]
\label{lemma:symmetric-poaa-lowerbound}
With symmetric agents, there exists no PIM such that 
\(
\text{PoA}_a < \mathbb{E} \left[ U_{\max} - \max_iU_{i,\min} \right],
\) where \(U_{\max}\) denotes the principal utility of the optimal answer among all candidate answers generated by agents, and \(U_{i,\min}\) denotes the principal utility of
agent \(i\)'s worst answer.
\end{lemma}

Lemma \ref{lemma:symmetric} illustrates a worst-case scenario in which the principal cannot access a good answer. This occurs when a super agent, not aligned with the principal, strategically submits low-utility answers that the principal is forced to accept, resulting in expected utility arbitrarily close to zero—even when much better answers exist.

Lemma \ref{lemma:symmetric-poaa-lowerbound} indicates that when the utility of agents is negatively correlated with that of the principal, agents tend to act selfishly and adversarially, submitting answers that harm the principal’s objective. Introducing independence addresses this issue by decoupling their incentives, thus reducing strategic misalignment. However, in our framework, the principal and agent are reacted based on same questions, therefore, assuming independent utility doesn't make sense. Therefore, these lemmas establish the necessity of the second and third part of Assumption \ref{assump:structure}.

\section{Computational Resources}\label{Sec: Resources}
Most of the computational cost occurs in Stage 1, where candidate answers are generated. Table \ref{tab:command_infer_cost} reports the mean number of model calls and generated tokens per question required to obtain candidate trajectories under 16 roll-outs. This cost scales linearly with both the number of questions and the number of roll-outs. In contrast, Stage 2 (delegation) is comparatively inexpensive - approximately 1.511 seconds per iteration per question. No additional infrastructure is needed to coordinate multiple agents.

Importantly, unlike traditional fine-tuning pipelines that require hours to days of compute, the ALIGN mechanism is fully training-free. Full-model fine-tuning typically demands thousands of GPU hours and can cost between $10,000$ and over $35,000$ per run \cite{liu2024deepseek}. Even parameter-efficient fine-tuning (PEFT) approaches such as FinLoRA \cite{wang2025finlora} still require several hours to days to update model weights. In contrast, ALIGN operates without modifying model parameters or performing any gradient updates, thereby eliminating the need for a training phase entirely.

\begin{table}[ht]
  \centering
  \caption{Inference cost of generating candidate answers stage of \textsc{ALIGN} on GSM8K: mean model calls and generated tokens per question}
  \label{tab:command_infer_cost}
  \resizebox{\linewidth}{!}{%
  \begin{tabular}{lcccc}
    \toprule
    & Mistral-8B-Instruct & Zephyr-7B & Phi3-Mini-4K-Instruct & Falcon-7B-Instruct \\
    \midrule
    Avg.\ calls         & 98.45 &  73.30 & 102.82 & 57.60 \\
    Avg.\ generated tokens & 144.80K & 81.74K & 263.76K& 51.81K \\
    \bottomrule
  \end{tabular}}
\end{table}
\end{document}